\documentclass{article}
\usepackage{amsthm}

\usepackage{amsmath, amssymb}
\usepackage{url}
\usepackage{rpa-misc}
\usepackage{macro_oren}

\usepackage{lipsum}
\usepackage{color}

\usepackage{times}

\usepackage{graphicx}
\usepackage{subfigure}

\usepackage{natbib}

\usepackage{algorithm}
\usepackage{algorithmic}

\newtheorem{thm}{Theorem}
\newtheorem{lemma}{Lemma}
\newtheorem*{dfn}{Definition}

\usepackage{hyperref}

\usepackage[accepted]{icml2014}

\icmltitlerunning{Learning Ordered Representations with Nested Dropout}

\begin{document} 

\twocolumn[
\icmltitle{Learning Ordered Representations with Nested Dropout}

\icmlauthor{Oren Rippel}{rippel@math.mit.edu}
\icmladdress{Graduate Program in Applied Mathematics, MIT; Harvard University}

\icmlauthor{Michael A. Gelbart}{mgelbart@seas.harvard.edu}
\icmladdress{Graduate Program in Biophysics,
Harvard University}

\icmlauthor{Ryan P. Adams}{rpa@seas.harvard.edu}
\icmladdress{School of Engineering and Applied Sciences,
Harvard University}

\vskip 0.3in
]

\begin{abstract} 

In this paper, we study ordered representations of data in which different dimensions have different degrees of importance. To learn these representations we introduce \emph{nested dropout}, a procedure for stochastically removing coherent nested sets of hidden units in a neural network. We first present a sequence of theoretical results in the simple case of a semi-linear autoencoder.  We rigorously show that the application of nested dropout enforces identifiability of the units, which leads to an exact equivalence with PCA.  We then extend the algorithm to deep models and demonstrate the relevance of ordered representations to a number of applications.  Specifically, we use the ordered property of the learned codes to construct hash-based data structures that permit very fast retrieval, achieving retrieval in time logarithmic in the database size and independent of the dimensionality of the representation. This allows codes that are hundreds of times longer than currently feasible for retrieval.  We therefore avoid the diminished quality associated with short codes, while still performing retrieval that is competitive in speed with existing methods.  We also show that ordered representations are a promising way to learn adaptive compression for efficient online data reconstruction. 
\end{abstract}

\section{Introduction}


The automatic discovery of representations is an increasingly important aspect of machine learning, motivated by a variety of considerations.  For example, feature extraction is often a critical first step for supervised learning procedures. Representation learning enables one to avoid explicit feature engineering; indeed, approaches to deep feature learning have often found representations that outperform their hand-crafted counterparts \citep[e.g.,][]{cnn,hinton_science,ae, coates_analysis}.   In other situations, unsupervised representation learning is useful for finding low-dimensional manifolds for visualization \citep[e.g.,][]{isomap, lle, t-sne}.  There has also been increasing interest in exploiting such representations for information retrieval, leveraging the ability of unsupervised learning to discover compact and concise codes that can be used to build efficient data structures \citep[e.g.,][]{spectral_hashing, semantic, deep_ae}.

One frustration associated with current representation learning techniques, however, is redundancy from non-identifiability in the resulting encoder/decoder.  That is, under standard models such as autoencoders, restricted Boltzmann machines, and sparse coding, any given solution is part of an \emph{equivalence class} of solutions that are equally optimal.  This class emerges from the invariance of the models to various transformations of the parameters.  Permutation is one clear example of such a transformation, leading to a combinatorial number of equivalent representations for a given dataset and architecture.  There exist many more kinds of redundancies as well; the optimality of an autoencoder solution is preserved under any invertible linear transformation of the innermost set of weights \citep{autoassociation}. This degeneracy also poses a difficulty when comparing experiments, due to the lack of repeatability: a solution attained by the optimization procedure is extremely sensitive to the choice of initialization. 

This large number of equivalent representations has an advantage, however: it provides flexibility in architecture design. This freedom allows us to impose desirable structural constraints on the learned representations, without compromising their expressiveness.  These constraints can imbue a number of useful properties, including the elimination of permutation non-identifiability.  In this work we propose one such structural constraint: we specify \emph{a priori} the quantity of information encapsulated in each dimension of the representation. This choice allows us to order the representation dimensions according to their information content.

The intuition behind our proposed approach to learning ordered representations is to train models such that the information contained in each dimension of the representation decreases as a function of the dimension index, following a pre-specified decay function. To this end, we introduce the \emph{nested dropout} algorithm.  As with the original dropout formulation \citep{hinton_do}, nested dropout applies a stochastic mask over models.  However, instead of imposing an independent distribution over each individual unit in a model, it assigns a distribution over nested subsets of representation units. More specifically, given a representation space of dimension~$K$, we define a distribution~$p_B\cd$ over the representation index subsets~${S_b=\{1,\ldots, b \}}$,~${b=1,\ldots, K}$. This has the property that if the $j$-th unit appears in a particular mask, then so do all ``earlier'' units ${1,\ldots, j-1}$, allowing the $j$-th unit to depend on them. This nesting leads to an inherent ordering over the representation dimensions. The distribution $p_B\cd$ then governs the information capacity decay by modulating the relative frequencies of these masks.

We motivate such ordered representations in several ways:

\paragraph{Identifiability} As discussed above, many current representation learning techniques suffer from non-identifiability of the solutions. We can remedy this by introducing strict representation ordering, which enforces distinguishability. We rigorously demonstrate this for the simple case of a semi-linear autoencoder. We prove that the application of nested dropout leads to a significant reduction in the solution space complexity without harming the solution quality. Under an additional weak constraint, we further prove that the model has a single and unique global optimum. We show that this solution is exactly the set of eigenvalues of the covariance matrix of the data, ordered by eigenvalue magnitude. This demonstrates exact equivalence between semi-linear nested dropout autoencoders and principal component analysis (PCA).

\paragraph{Fast retrieval} Current information retrieval procedures suffer from an intrinsic tradeoff between search speed and quality: representation dimensionality and dataset size must be sacrificed to gain search tractability (\citet{hashes} offers an excellent overview of modern retrieval procedures). Given a query datum, a na\"{i}ve brute force retrieval based on Hamming distance requires a linear scan of the database, which has complexity $\mcO\left(KN\right)$ where $K$ is the code length and $N$ the database size.  Semantic hashing \citep{semantic} retrieves examples within a Hamming neighborhood of radius $R$ by directly scanning through all memory locations associated with them. This results in retrieval time complexity~$\mcO(\binom{K}{R})$. While this is independent of the database size, it grows rapidly in $K$ and therefore is computationally prohibitive even for codes tens of bits long; code length of 50 bits, for example, requires a petabyte of memory be addressed.  Moreover, as the code length increases, it becomes very likely that many queries will not find any neighbors for any feasible radii. Locality sensitive hashing \citep{lsh} seeks to preserve distance information by means of random projections; however, this can lead to very inefficient codes for high input dimensionality. 

By imposing an ordering on the information represented in a deep model, we can learn hash functions that permit efficient retrieval.  Because the importance of each successive coding dimension decays as we move through the ordering, we can naturally construct a binary tree data structure on the representation to capture a coarse-to-fine notion of similarity.  This allows retrieval in time that is logarithmic with the dataset size and \emph{independent} of the representation space dimensionality: the retrieval procedure adaptively selects the minimum number of code bits required for resolution. This enables very fast retrieval on large databases without sacrificing representation quality: we are able to consider codes hundreds of times longer than currently feasible with existing retrieval methods. For example, we perform retrieval on a dataset of a million entries of code length $2048$ in an average time of $200 \mu$s per query---15,000 times faster than a linear scan or semantic hashing. 

\paragraph{Adaptive compression} Ordered representations can also be used for ``continuous-degradation'' lossy compression systems: they give rise to a continuous range of bitrate/quality combinations, where each additional bit corresponds to a small incremental increase in quality. This property can in principle be applied to problems such as video streaming. The representation only needs to be encoded a single time; then, users of different bandwidths can be adaptively sent codes of different length that exactly match their bitrates. The inputs can then be reconstructed optimally for the users' channel capacities.
\label{introduction}

\section{Ordering with nested dropout}

Dropout \citep{hinton_do} is a regularization technique for neural networks that adds stochasticity to the architecture during training. At each iteration, unbiased coins are flipped independently for each unit in the network, determining whether it is ``dropped'' or not. Every dropped unit is deleted from the network for that iteration, and an optimization step is taken with respect to the resulting network. 

Nested dropout diverges from this in two main ways. First, we only drop units in the representation space. Second, instead of flipping independent coins for different units, we instead assign a prior distribution $p_B\cd$ over the representation \emph{indices} $1,\ldots,K$. We then sample an index~${b\sim p_B\cd}$ and drop units ${b+1,\ldots,K}$. The sampled units then form nested subsets: if unit $j$ appears in a network sample, then so do units $1,\ldots,j-1$. This nesting results in an inherent importance ranking of the representation dimensions, as a particular unit can always rely on the presence of its predecessors. For $p_B\cd$ we select a geometric distribution: ${p_B(b)=\rho^{b-1} (1-\rho)}$. We make this choice due to the exponential decay of this distribution and its memoryless property (see Section \ref{implementation}).

Our architecture resembles an autoencoder in its parametric composition of an encoder and a decoder. We are given a set of~$N$ training examples~${\{\rmby_{n}\}_{n=1}^{N}}$ lying in space~${\scrY\subseteq\reals^D}$~. We then transform the data into the \emph{representation space}~${\scrX\subseteq\reals^K}$ via a parametric transformation~${\bbf_{\bTheta}:\scrY\to\scrX}$. We denote this function as the \emph{encoder}, and label the representations as~${\{\rmbx_n\}_{n=1}^N\subset\scrX}$. The \emph{decoder} map~${\bg_{\bPsi}:\scrX\to\scrY}$ then reconstructs the inputs from their representations as~$\{\hat{\rmby}_n\}_{n=1}^N$.

\paragraph{A single nested dropout sample}
Let us assume that we sample some~${b\sim p_B\cd}$ and drop the last~${K-b}$ representation units; we refer to this case as the \emph{$b$-truncation}. This structure is equivalent to an autoencoder with a representation layer of dimension $b$. For a given representation~${\rmbx\in\reals^K}$, we define~$\rmbx\up{b}$ as the truncation of vector~$\rmbx$, namely a copy of~$\rmbx$ where the last~${K-b}$ units are removed, or are equivalently set to~$0$.

Denoting the reconstruction of the $b$-truncation as~${\hat{\rmby}\up{b}=\bg_{\bPsi}(\bbf_{\bTheta}(\rmby)\up{b})}$, the reconstruction cost function associated with a $b$-truncation is then
\begin{equation}
C\up{b}(\bTheta,\bPsi) = \frac{1}{N}\sum_{n=1}^N \mathscr{L}\left(\rmby_n, \hat{\rmby}_{n\downarrow b}\right)\;.
\end{equation}
In this work, we take the reconstruction loss $\mathscr{L}(\cdot,\cdot)$ to be the $L_2$ norm. Although we write this cost as a function of the full parametrization $(\bTheta, \bPsi)$, due to the truncation only a subset of the parameters will contribute to the objective.

\paragraph{The nested dropout problem}

Given our distribution~$p_B\cd$, we consider the mixture of the different $b$-truncation objectives: 
\begin{align}
C(\bTheta,\bPsi)\!&=\!\bbE_B\left[C\up{b}(\bTheta,\bPsi)\right]\!=\!
   \sum_{b=1}^K p_B(b)C\up{b}(\bTheta,\bPsi).
\end{align}
We formulate the \emph{nested dropout} problem as the optimization of this mixture with respect to the model parameters:
\begin{eqnarray}
(\bTheta^*,\bPsi^*)=\arg\min_{\bTheta,\bPsi} \;\; C(\bTheta,\bPsi)\;.\label{ndo_problem}
\end{eqnarray}
\subsection{Interpretation}

Nested dropout has a natural interpretation in terms of information content in representation units. It was shown by \citet{ae} that training an autoencoder corresponds to maximizing a lower bound on the mutual information~$\mcI(\by;\bx)$ between the input data and their representations. Specifically, the objective of the $b$-truncation problem can be written in the form
\begin{align}
C\up{b}(\bTheta,\bPsi) \approx \bbE_{\rmby} \left[ -\log p_{\bY | \bX\up{b}} \left(\rmby\,|\,\bbf_{\bTheta}(\rmby)\up{b}\,;\,\bPsi \right) \right]\label{ae_bound}
\end{align}
where we assume our data are sampled from the true distribution~$\pY\cd$. The choice~$p_{\bY | \bX} \left(\rmby\,|\,\rmbx\,;\,\bPsi\right)=\mcN(\rmby\,;\,\bg_{\bPsi}(\rmbx),\sigma^2 \mathbb{I}_D )$, for example, leads to the familiar autoencoder~$L_2$~reconstruction penalty.

Now, define $\tilde{\mcI}_b(\by;\bx):=-C\up{b}(\bTheta,\bPsi)\leq\mcI(\by;\bx)$ as the approximation of the true mutual information which we maximize for a given $b$. Then we can write the (negative) nested dropout problem in the form of a telescopic sum:
\begin{align}
-C(\bTheta,\bPsi) &= \sum_{b=1}^K p_B(b)\tilde{\mcI}_b(\by;\bx)\label{info} \\
&=\tilde{\mcI}_1(\by;\bx)+\sum_{b=2}^{K} \left[ F_B(K) - F_B(b-1) \right] \Delta_b\nonumber \;,
\end{align}
where~${F_B(b)=\sum_{b'=1}^b p_B(b')}$~is the cumulative distribution function of $p_B\cd$~, and $\Delta_b := \tilde{\mcI}_b(\by;\bx) - \tilde{\mcI}_{b-1}(\by;\bx)$ is the marginal information gained from increasing the representation dimensionality from $b$ units to~${b+1}$. 

This formulation provides a connection between the nested dropout objective and the optimal distribution of information across the representation dimensions. Note that the coefficients~${F_B(K) - F_B(b)}$~of the marginal mutual information are positive and monotonically decrease as a function of~$b$~regardless of the choice of distribution~$p_B\cd$. This establishes the ordering property intuitively sought by the nested dropout idea. We also see that if for some $b$ we have~${p_B(b)=0}$, i.e., index $b$ has no support under $p_B\cd$, then the ordering of representation dimensions~$b$ and~${b-1}$ no longer matters. If we set~${p_B(1)=0,\ldots,p_B(K-1)=0}$ and ~${p_B(K)=1}$, we recover the original order-free autoencoder formulation for~$K$ latent dimensions. In order to achieve strict ordering, then, the only assumption we must make is that~$p_B\cd$ has support over all representation indices. Indeed, this will be a sufficient condition for our proofs in Section \ref{pca}. Equation (\ref{info}) informs us of how our prior choice of $p_B\cd$ dictates the optimal information allocation per unit.

\label{ndo}

\section{Exact recovery of PCA}

In this section, we apply nested dropout to a semi-linear autoencoder. This model has a linear or a sigmoidal encoder, and a linear decoder. The relative simplicity of this case allows us to rigorously study the ordering property implied by nested dropout. 

First, we show that the class of optimal solutions of the nested dropout autoencoder is a subset of the class of optimal solutions of a standard autoencoder. This means that introducing nested dropout does not sacrifice the quality of the autoencoder solution.  Second, we show that equipping an autoencoder with nested dropout significantly constrains its class of optimal solutions. We characterize these restrictions.  Last, we show that under an additional orthonormality constraint, the model features a single, unique solution that is exactly the set of~$K$~eigenvectors with the largest magnitudes arising from the covariance matrix of the inputs, ordered by decreasing eigenvalue magnitude. Hence this recovers the PCA solution exactly. This is in contrast to a standard autoencoder, which recovers the PCA solution up to an invertible linear map.

\subsection{Problem definitions and prior results}
\label{defns}

\paragraph{The standard linear autoencoder problem} Given our inputs, we apply the linear encoder~${\bbf_{\bTheta}(\rmby):=\bOmega\rmby+\bomega}$ with parameters~${\bOmega\in\reals^{K\times D}}$ and bias vector~${\bomega\in\reals^{K}}$ for~${K\leq D}$. Our proofs further generalize to sigmoidal nonlinearities applied to the output of the encoder, but we omit these for clarity. The decoder map~${\bg_{\bPsi}:\scrX\to\scrY}$ is similarly taken to be~${\bg_{\bPsi}(\rmbx):=\bGamma\rmbx+\bgamma}$ with parameters~${\bGamma\in\reals^{D\times K}}$ and~${\bgamma\in\reals^{D}}$. We also define the design matrices~$\bY$ and $\bX$ whose columns consist of the observations and their representations, respectively.

The reconstruction of each datum is then defined as the composition of the encoder and decoder maps. Namely, $\hat{\rmby}_n=\bGamma(\bOmega\rmby_n+\bomega)+\bgamma\; \forall n=1,\ldots,N$. A semi-linear autoencoder seeks to minimize the reconstruction cost
\begin{eqnarray}
C(\bTheta,\bPsi) &=& \sum_{n=1}^N \norm{\rmby_n-{\bg_{\bPsi}}(\bbf_{\bTheta}(\rmby_n)) }^2 \\
  &=&\norm{\bY-(\bGamma(\bOmega\bY+\bomega)+\bgamma) }_F^2
\end{eqnarray}
where by $\norm{\cdot}_F$ we denote the Frobenius matrix norm. From this point on, without loss of generality we assume that $\bomega=\bzero$, $\bgamma=\bzero$, and that the data is zero-centered. All our results hold otherwise, but with added shifting constants.

\paragraph{A single $b$-truncation problem} We continue to consider the $b$-truncation problem, where the last $K-b$ units of the representation are dropped. As before, for a given representation $\rmbx\in\reals^K$, we define $\rmbx\up{b}$ to be the truncation of vector $\rmbx$. Defining the truncation matrix~${\bJ_{m\rightarrow n}\in\reals^{n\times m}}$ as~${[\bJ_{m\rightarrow n}]_{ab}=\delta_{ab}}$, then~${\rmbx\up{b}=\bJ_{K\rightarrow b}\rmbx}$. The decoder is then written as ${{\bg_{\bPsi}}\up{b}(\rmbx\up{b})=\bGamma\up{b}\rmbx\up{b}}$, where we write $\bGamma\up{b}=\bGamma\trJ{K}{b}^T$ in which the last $K-b$ columns of $\bGamma$ are removed. The reconstruction cost function associated with a $b$-truncation is then
\begin{equation}
C\up{b}(\bTheta\up{b},\bPsi\up{b})=\norm{\bY-\Gamma\up{b}\bX\up{b} }_F^2\;.
\end{equation}
We define ${(\bTheta\up{b}^*,\bPsi\up{b}^*)=\arg\min_{\bTheta\up{b},\bPsi\up{b}}\bC\up{b}(\bTheta\up{b},\bPsi\up{b})}$ to be an optimal solution of the $b$-truncation problem; we label the corresponding optimal cost as~$C\up{b}^*$. Also, let ${\bV_{\bY}=\bY\bY^T}$ be (proportional to) the empirical covariance matrix of~$\{\rmby_n\}_{n=1}^N$ with eigendecomposition~${\bV_{\bY}=\bQ\bSigma^2\bQ^T}$, where $\bSigma^2$ is the diagonal matrix constituting of the eigenvalues arranged in decreasing magnitude order, and~$\bQ$ the orthonormal matrix of the respective eigenvectors. Similarly, let $\bR$ be the orthonormal eigenvector matrix of $\bY^T\bY$, arranged by decreasing order of eigenvalue magnitude. 

The $b$-truncation problem exactly corresponds to the original semi-linear autoencoder problem, where the representation dimension is taken to be $b$ in the first place. As such, we can apply known results about the form of the solution of a standard autoencoder. It was proven in \citet{autoassociation} that this optimal solution must be of the form
\begin{align}
\bX_b^* &= \bT_b \bSigma\up{b} \bR^T &
\bGamma_b^* &= \bQ\up{b}\bT^{-1}_b\label{soln_form}
\end{align}
where ${\bT_b\in\reals^{b\times b}}$ is an invertible matrix, ${\bSigma\up{b}=\trJ{K}{b}\bSigma\in\reals^{b\times D}}$ the matrix with the $b$ largest-magnitude eigenvalues, and ${\bQ\up{b}=\bQ \trJ{K}{b}^T\in\reals^{D\times b}}$ the matrix with the $b$ corresponding eigenvectors. This result was established for an autoencoder of representation dimension $b$; we reformulated the notation to suit the nested dropout problem we define in the next subsection.

It can be observed from Equation~(\ref{soln_form}) that the semi-linear autoencoder has a strong connection to PCA. An autoencoder discovers the eigenvectors of the empirical covariance matrix of $\{\rmby_n\}_{n=1}^N$ corresponding to its $b$ eigenvalues of greatest magnitude; however, this is up to to an invertible linear transformation. This class includes rotations, scalings, reflections, index permutations, and so on. This non-identifiability has an undesirable consequence: it begets a huge class of optimal solutions.

\paragraph{The nested dropout problem} We now introduce the nested dropout problem. Here, we assign the distribution~$b\sim p_B(\cdot)$ as a prior over $b$-truncations. For our proofs to hold our only assumption about this distribution is that it has support over the entire index set, i.e., ${p_B(b)>0,\;\forall b=1,\ldots,K}$. To that end, we seek to minimize the nested dropout cost function, which we define as the mixture of the $K$ truncated models under~$p_B\cd$:
\begin{align}
C(\bTheta,\bPsi) &=\bbE_B \left[\norm{\bY-\Gamma\up{b}\bX\up{b} }_F^2\right]\\
  &= \sum_{b=1}^K p_B(b)\norm{\bY-\Gamma\up{b}\bX\up{b}}_F^2\;.\label{ndo_objective}
\end{align} 

\subsection{The nested dropout problem recovers PCA exactly}\label{pca_subsection}

Below we provide theoretical justification for the claims made in the beginning of this section. All of the proofs can be found in the appendix of this paper.

\begin{thm}\label{everysoln}
Every optimal solution of the nested dropout problem is necessarily an optimal solution of the standard autoencoder problem.
\end{thm}

\begin{dfn}
We define matrix ${\bT\in\reals^{K\times K}}$ to be \emph{commutative in its truncation and inversion} if each of its leading principal minors ${\trJ{K}{b}\bT\trJ{K}{b}^T},\; {b=1,\ldots,K}$ is invertible, and the inverse of each of its leading principal minors is equal to the leading principal minor of the inverse $\bT^{-1}$, namely
\begin{align}
\trJ{K}{b}\bT^{-1}\trJ{K}{b}^T &= (\trJ{K}{b}\bT\trJ{K}{b}^T)^{-1}\;.
\end{align}
\end{dfn}

The below theorem, combined with Lemma \ref{lemma}, establishes tight constraints on the class of optimal solutions of the nested dropout problem. For example, an immediate corollary of this is that $\bT$ cannot be a permutation matrix, as for such a matrix there must exist some leading principal minor that is not invertible. 

\begin{thm}\label{soln_constraints}
Every optimal solution of the nested dropout problem must be of the form
\begin{align}
\bX^* &= \bT \bSigma \bR^T &
\bGamma^* &= \bQ \bT^{-1}\;,
\end{align}
for some matrix ${\bT\in\reals^{K\times K}}$ that is commutative in its truncation and inversion.
\end{thm}

Denote the column and row submatrices, respectively as $\bA_b=[ T_{1b},  \ldots, T_{(b-1), b}]^T$ and $\bB_b=[ T_{b1}, \ldots, T_{b, (b-1)}]^T$. 

\begin{lemma}\label{lemma}
Let $\bT\in\reals^{K\times K}$ be commutative in its truncation and inversion. Then all the diagonal elements of $\bT$ are nonzero, and for each $b=2,\ldots,K$, either~${\bA_b=\textbf{0}}$ or~${\bB_b=\textbf{0}}$.
\end{lemma}

In the result below we see that nested dropout coupled with an orthonormality constraint effectively eliminates non-identifiability. The added constraint pins down any possible rotations and scalings.

\begin{thm}
Under the orthonormality constraint ${\bGamma^T\bGamma=\mathbb{I}_K}$, the nested dropout problem features a unique global optimum, and this solution is exactly the set of the~$K$ top eigenvectors of the covariance of $\bY$, ordered by eigenvalue magnitude. Namely, ${\bX^*=\bSigma\bR^T}, {\bGamma^*=\bQ}$.
\end{thm}
\label{pca}

\section{Training deep models with nested dropout}

\begin{figure}[t!]
    \subfigure[\small Without unit sweeping]{
    \centering
    \includegraphics[width=0.47\columnwidth]{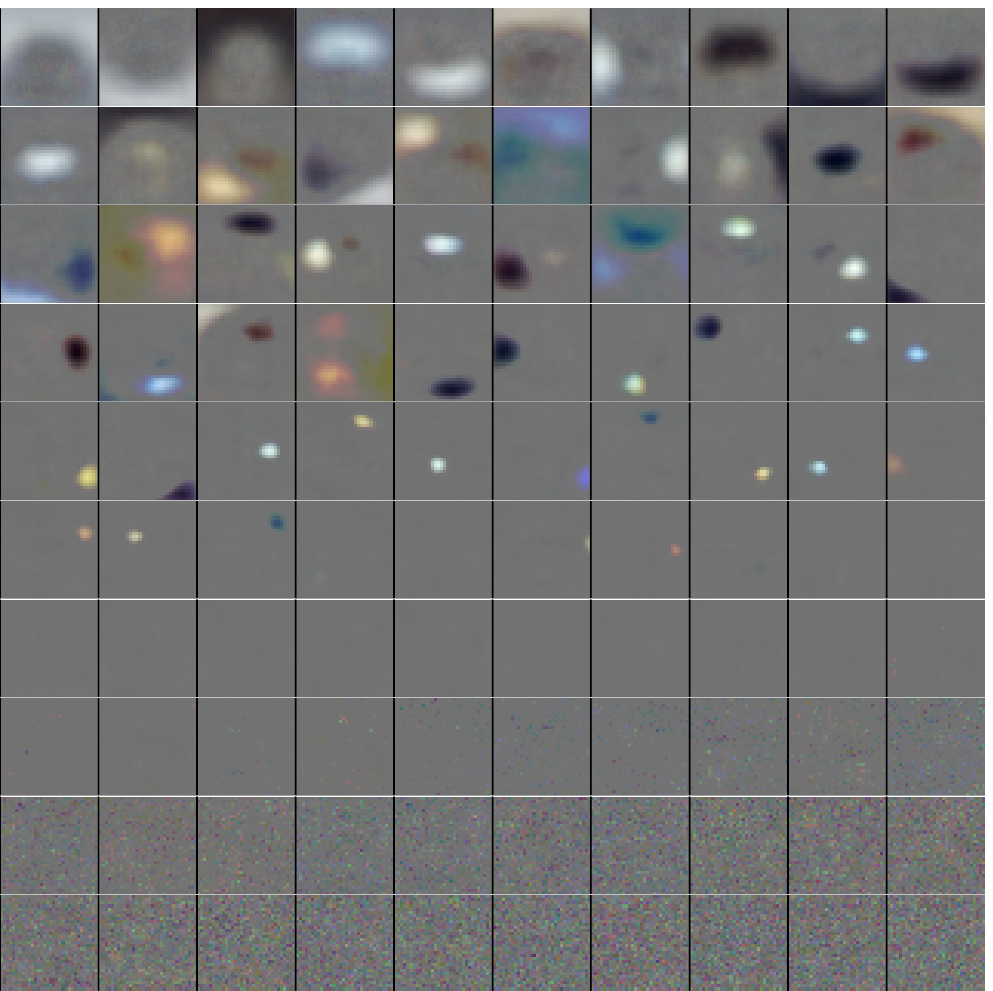}
    \label{sweep_none}}
    \hfill\subfigure[\small With unit sweeping]{
    \centering
    \includegraphics[width=0.47\columnwidth]{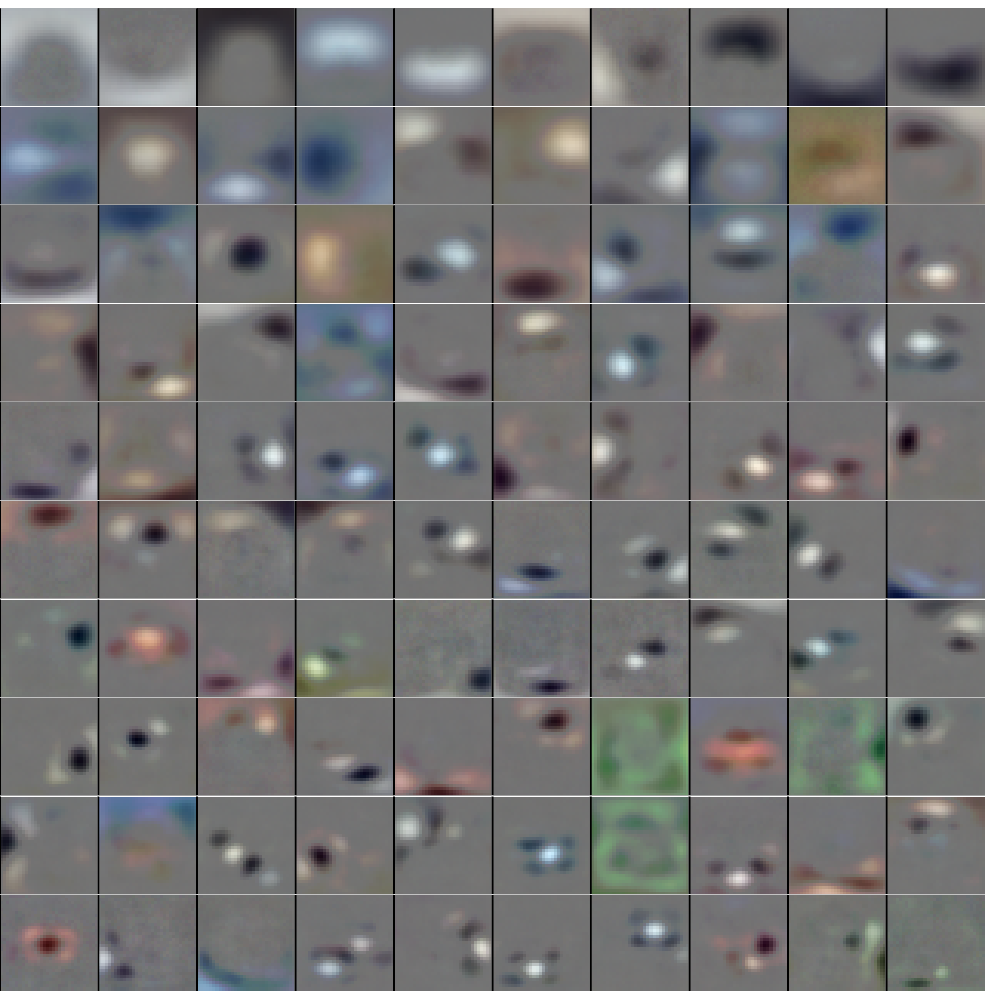}
    \label{sweep_yes}}
    \captionv{The 100 filters learned by a binarized 3072-100-3072 nested dropout autoencoder on the raw CIFAR-10 pixels where $p_B\cd$ is a geometric distribution with rate of 0.9. For this rate, the probability of sampling any index greater than 50 is $\approx 0.005$, and the probability of sampling the 100th unit is $\approx 0.00003$: it is very unlikely to ever sample these without unit sweeping. Note the increase of filter fineness as a function of the index.}
    \label{sweeping}
\end{figure}

In this section we discuss our extension of the nested dropout approach to deep architectures. Specifically, we applied this to deep autoencoders having tens of millions of parameters, which we trained on the 80 Million Tiny Images (80MTI) dataset \citep{80mti} on a cluster of GPUs. Training models with nested dropout introduces a number of unconventional technical challenges. In the proceeding sections we describe these challenges, and discuss strategies to overcome them.

We first describe our general architecture and optimization setup. The 80MTI are 79,302,017 color images of size ${32 \times 32}$.  We pre-processed the data by subtracting from each pixel its mean and normalizing by its variance across the dataset. We optimize our models with the nonlinear conjugate gradients algorithm and select step sizes using a strong Wolfe conditions line search. For retrieval-related tasks, we seek to produce binary representations. In light of this we use rectified linear units for all nonlinarities in our encoder, as we find this leads to better binarized representation (see Subsection \ref{binarization}). \citet{rectifier} features an in-depth discussion and motivation of rectified linear units. We train for 2 epochs on minibatches of size 10,000. We inject noise to promote robustness, as in \cite{ae}; namely, with probability 0.1 we independently corrupt input elements to 0. For all layers other than the representation layer, we apply standard dropout with probability 0.2. At each iteration, we sample nested dropout truncation indices for each example in our minibatch, and take a step with respect to the corresponding network mask. 

\subsection{Unit sweeping for decaying gradients}
By the virtue of the decaying distribution $p_B\cd$, it becomes increasingly improbable to sample higher representation indices during training. As such, we encounter a phenomenon where gradient magnitudes vanish as a function of representation unit index. This curvature pathology, in its raw formulation, means that training representation units of higher index can be extremely slow.

In order to combat this effect, we develop a technique we call \emph{unit sweeping}. The idea stems from the observation that the covariance of two latent units sharply decreases as a function of the of the difference of their indices. When $p_B\cd$ is a geometric distribution, for example, the probability of observing both units~$i$ and~$j$ given that one of them is observed is ${\mathbb{P}[b\!\geq\!\max (i, j) \given b\!\geq\!\min (i, j)]}={\mathbb{P}[b\!\geq\! |i\!-\!j|]}={\rho^{-|i-j|}}$ by the memoryless property of the distribution. In other words, a particular latent unit becomes exponentially desensitized to values of units of higher index. As such, this unit will eventually converge during its training. Upon convergence, then, this unit can be fixed in place and its associated gradients can be omitted. Loosely speaking, this elimination reduces the ``condition number'' of the optimization. Applying this iteratively, we sweep through the latent units, fixing each once it converges. In Figure \ref{sweeping} we compare filters from training a nested dropout model with and without unit sweeping.

\subsection{Adaptive regularization coefficients}
The gradient decay as a function of representation index poses a difficulty for regularization. In particular, the ratio of the magnitudes of the gradients of the reconstruction and the regularization vanishes as a function of the index. Therefore, a single regularization term such as $\lambda \sum_{k=1}^K \norm{\bOmega_k}_{L_1}$ would not be appropriate for nested dropout, since the regularization gradient would dominate the high-index gradients. As such, the regularization must be decoupled as a function of representation index. For weight decay, for example, this would of the form $\sum_{k=1}^K \lambda_k \norm{\bOmega_k}_{L_1}$. Choosing the coefficients $\lambda_k$ manually is challenging, and to that end we assign them adaptively. We do this by fixing in advance the ratio between the magnitude of the reconstruction gradient and the regularization gradient, and choosing the $\lambda_k$ to satisfy this ratio requirement. This corresponds to fixing the relative contributions of the terms at each step in the optimization procedure.

\subsection{Code binarization}\label{binarization}
For the task of retrieval, we would like to obtain binary representations. Several binarization methods have been proposed in prior work \citep{semantic, deep_ae}. We have empirically achieved good performance by tying the weights of the encoder and decoder, and thresholding at the representation layer. Although the gradient itself cannot propagate past this threshold, some signal does: the encoder can be trained since it is linked to the decoder, and its modifications are then reflected in the objective. To attain fixed marginal distributions over the binarized representation units, i.e., ${x_k\sim\textrm{Bern}(\beta)}$ for ${k=1,\ldots, K}$, we compute the $\beta$ quantile for each unit, and use this value for thresholding.

\subsection{Code invariance}
We attain better retrieval results when we demand code invariance.  That is, when we require that similar examples map to similar codes. Inspired by \citet{contractive, higher_contractive}, we perform this regularization stochastically by perturbing each input $\rmby_n$ in our minibatch with some small ${\bvepsilon_n\sim \mcN(\bzero, \bar{\varepsilon}\mathbb{I}_D)}$, and introduce penalty term
\begin{equation}
\mcI = \frac{1}{N}\sum_{n=1}^N \frac{\norm{\bbf_{\bTheta}(\rmby_n+\bvepsilon_n)-\bbf_{\bTheta}(\rmby_n)}^2}{\norm{\bvepsilon_n}^2}\;.
\end{equation}
It can be shown via Taylor expansion that this corresponds to a stochastic formulation of the regularization of the Frobenius norm of the Jacobian of $\bbf_{\bTheta}\cd$.

\begin{figure}[t!]
\parbox[t]{\columnwidth}{
\centering
\includegraphics[width=0.19\columnwidth,height=0.19\columnwidth]{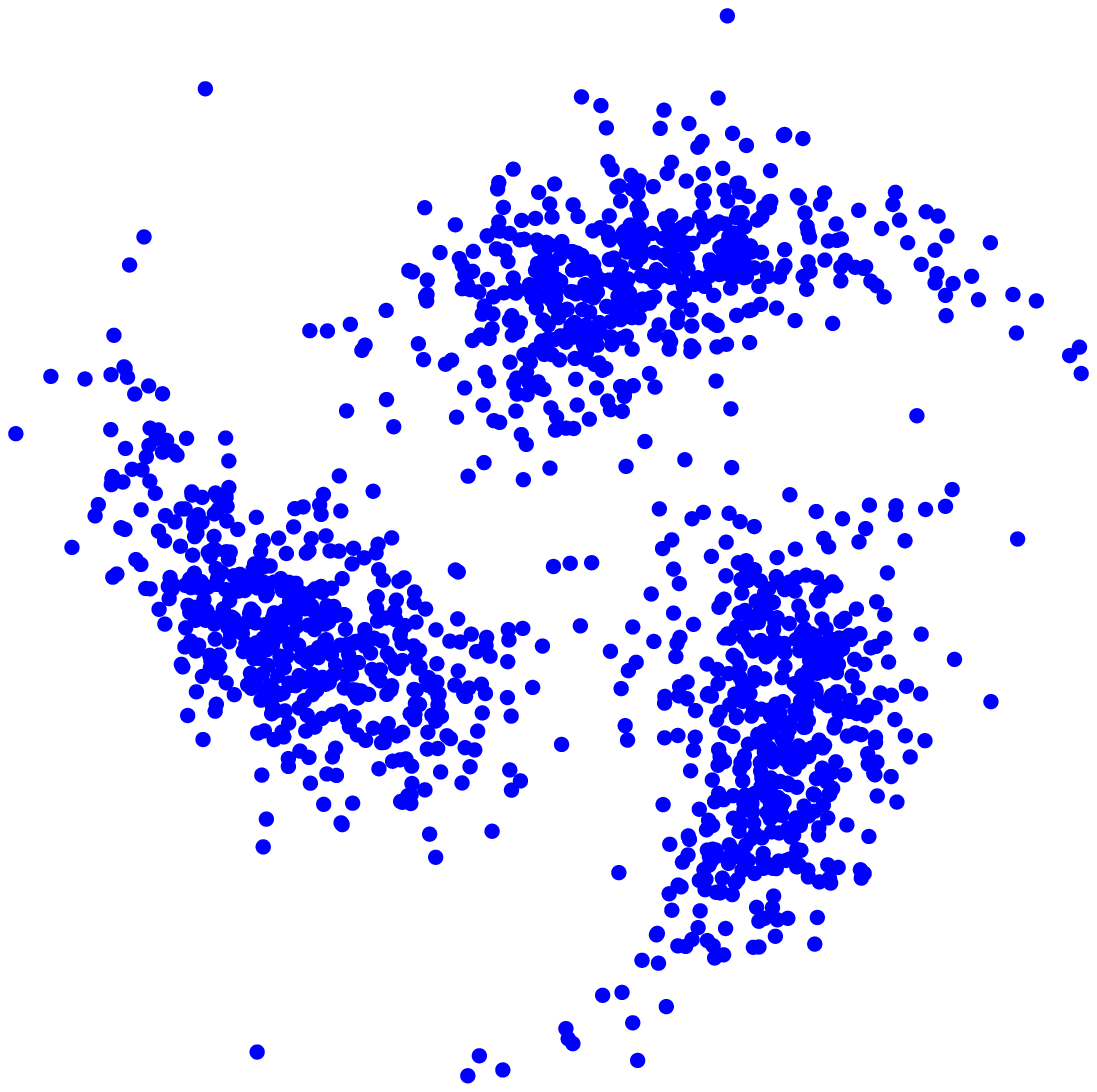}\hfill
\includegraphics[width=0.19\columnwidth,height=0.19\columnwidth]{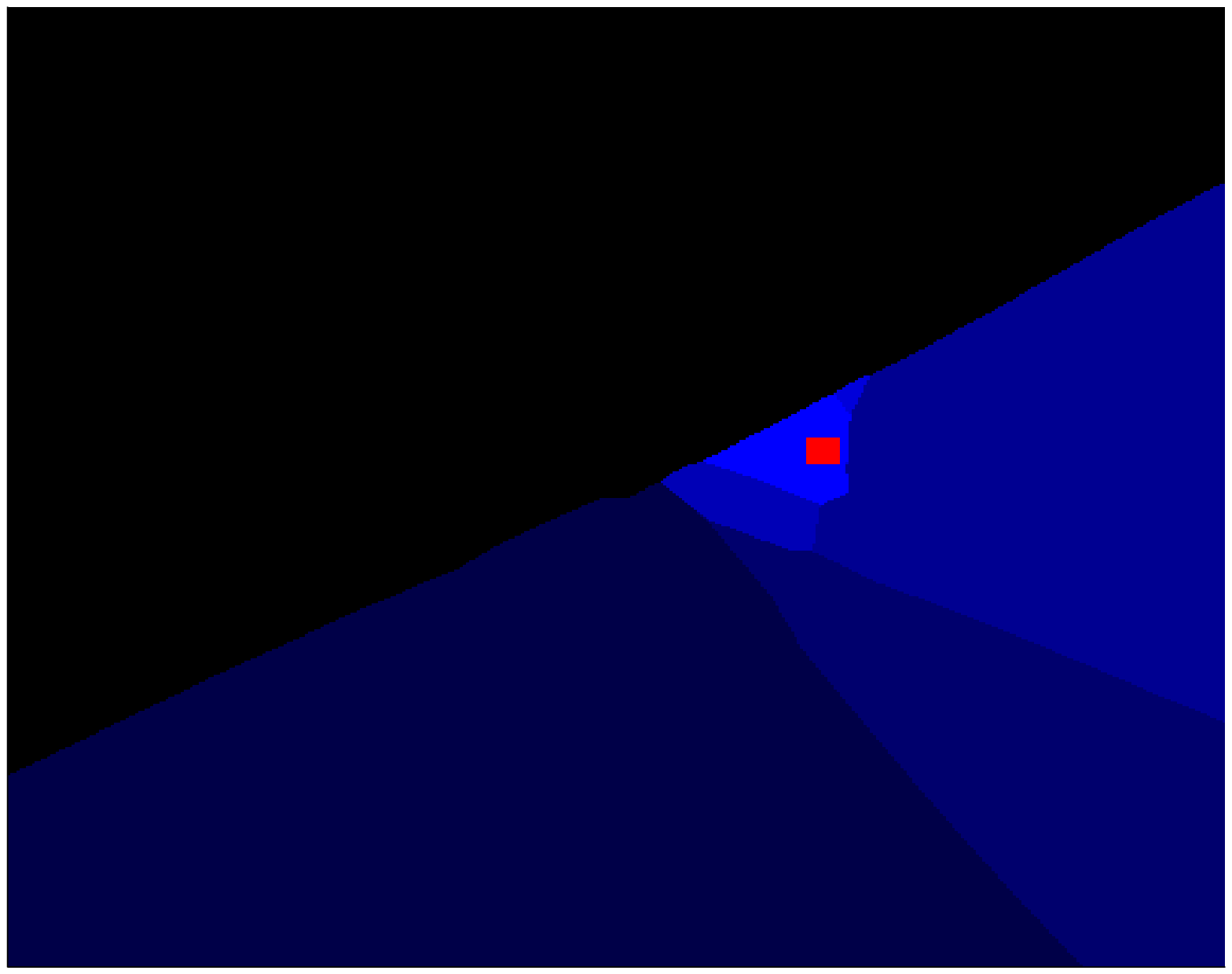}\hfill
\includegraphics[width=0.19\columnwidth,height=0.19\columnwidth]{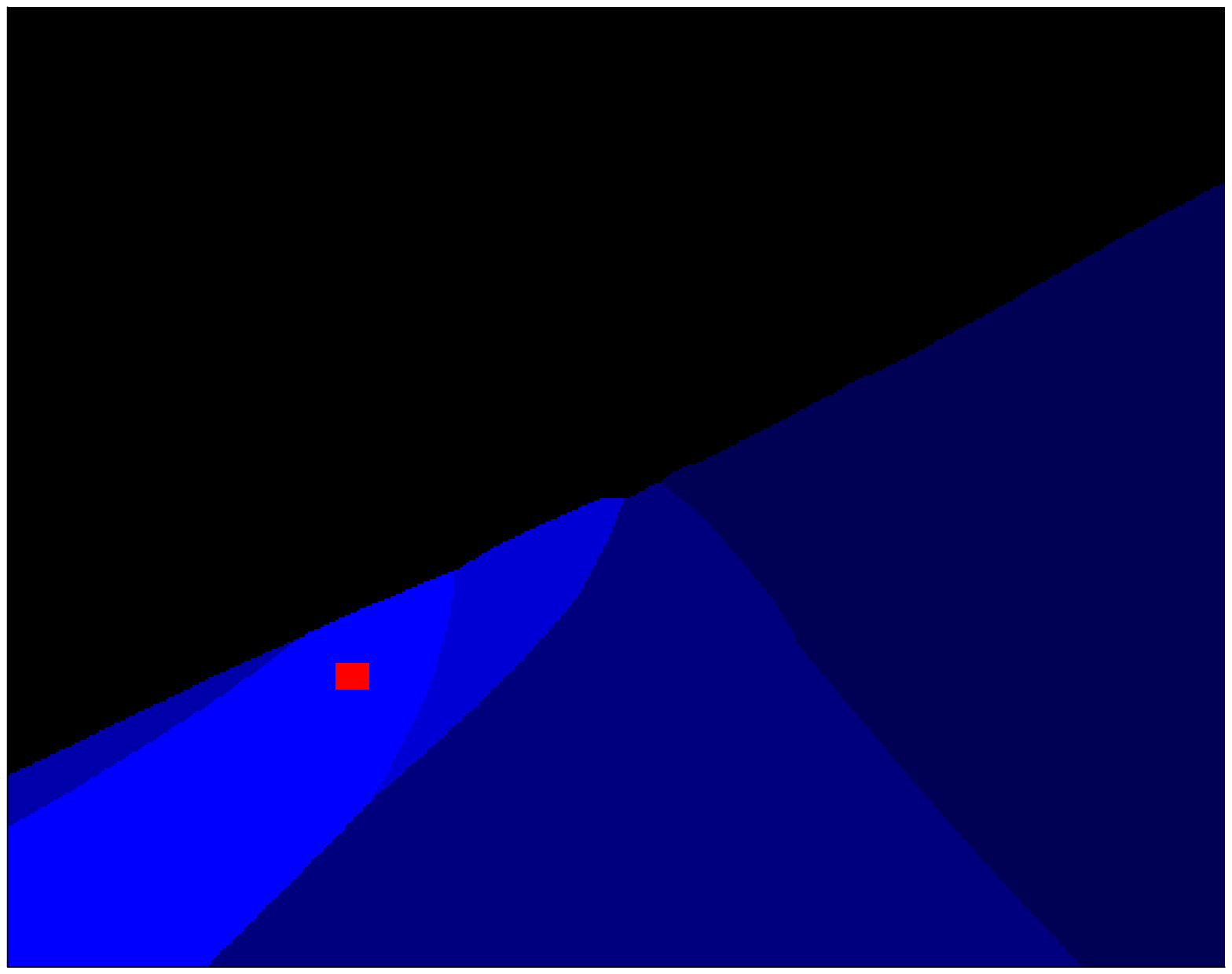}\vspace{0.05in}
\includegraphics[width=0.19\columnwidth,height=0.19\columnwidth]{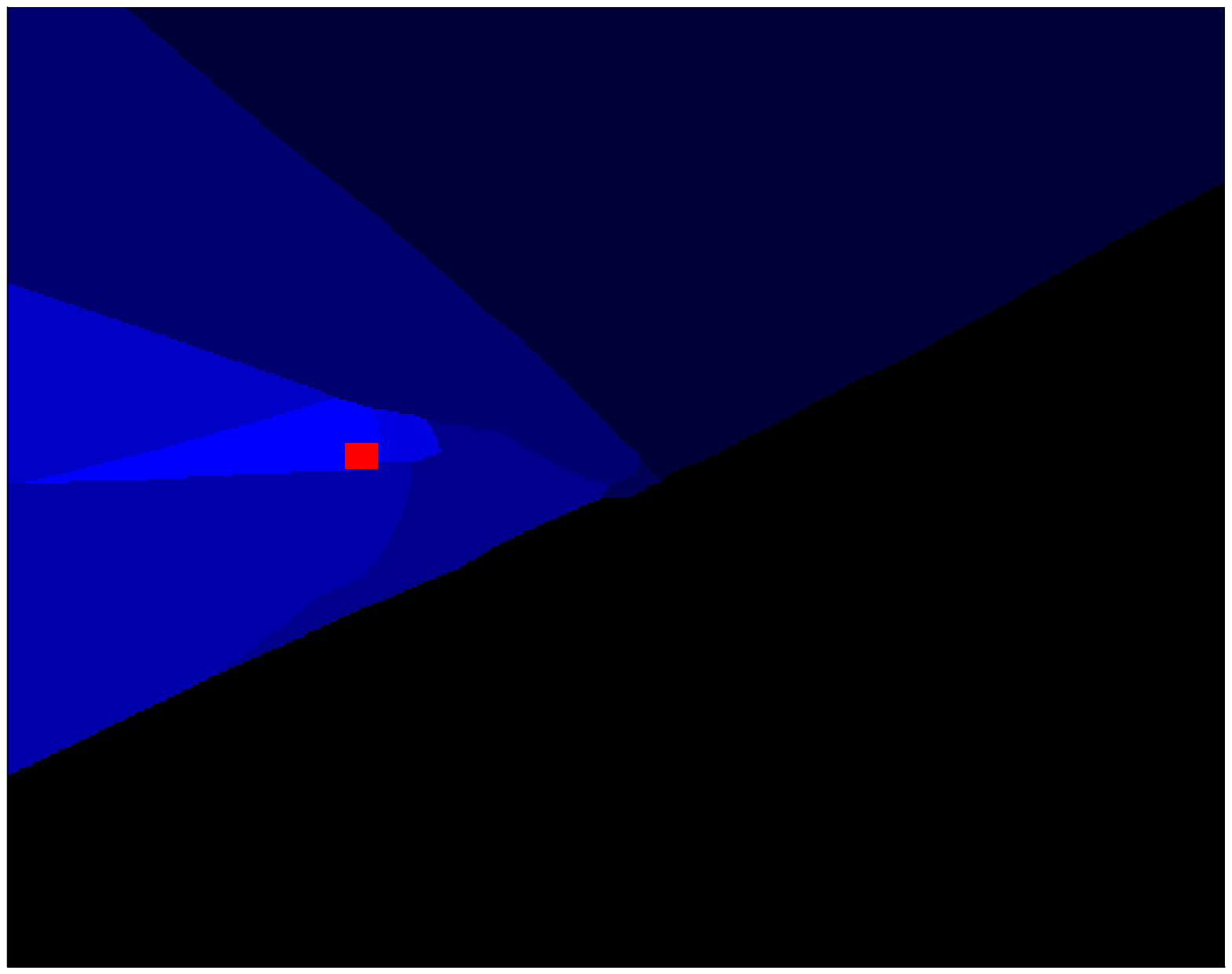}\hfill
\includegraphics[width=0.19\columnwidth,height=0.19\columnwidth]{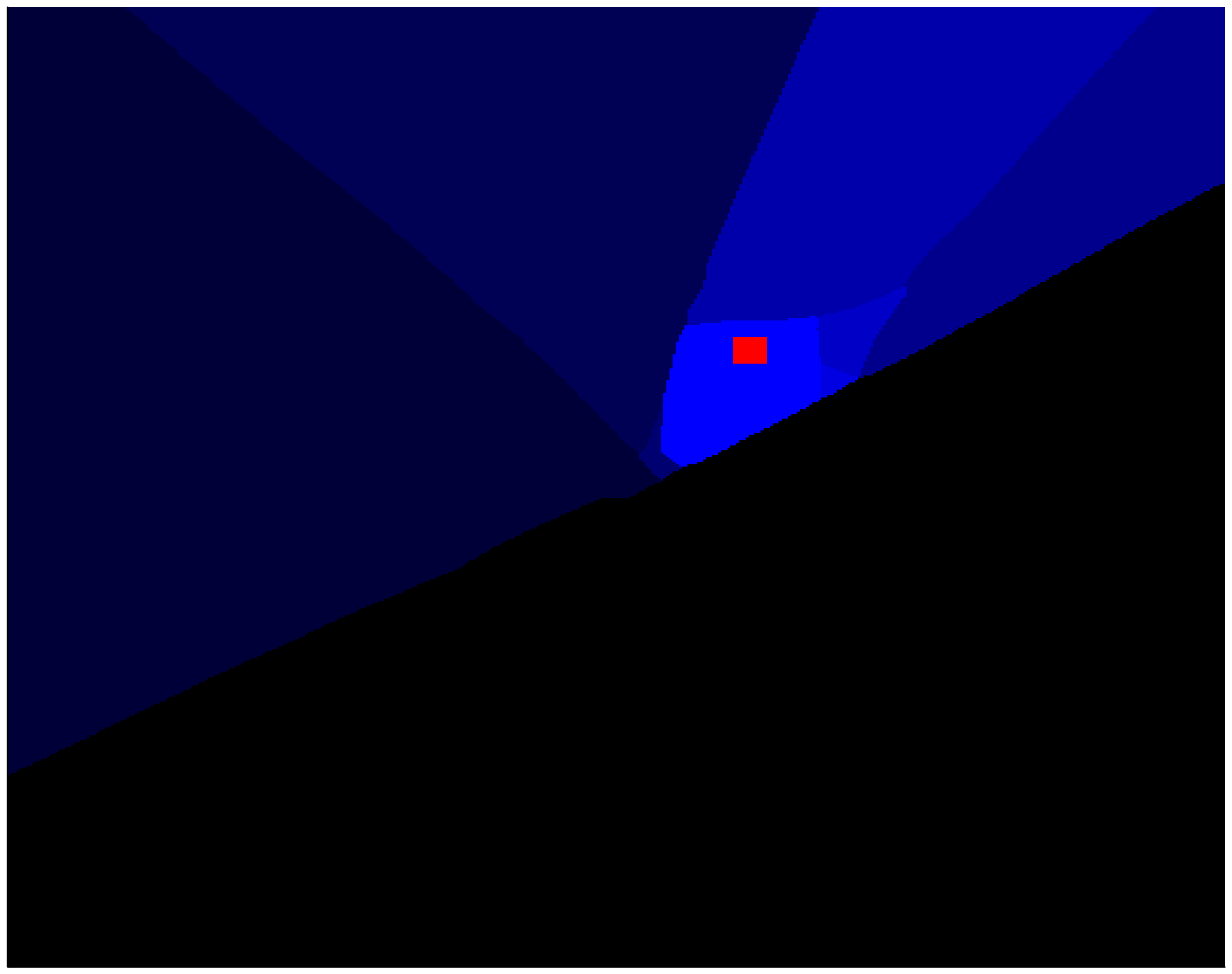}}
\captionv{Nested neighborhoods for various examples in a 2D toy problem. The synthetic pinwheel training data is on the top left. The trained model is a 2-16-32-16-2 nested dropout autoencoder. The red dots correspond to the retrieved queries. The shades of blue correspond to different nested neighborhoods for these queries, with color lightness signifying neighborhood depth.}
    \label{toy}
\end{figure}

\label{implementation}

\section{Retrieval with ordered binary codes}

In this section we discuss how ordered representations can be exploited to construct data structures that permit fast retrieval while at the same time allowing for very long codes.

\subsection{Binary tree on the representation space}

The ordering property, coupled with the ability to control information capacity decay across representation units, motivates the construction of a binary tree over large data sets. Each node in this tree contains pointers to the set of examples that share the same path down the tree up to that point. Guaranteeing that this tree is balanced is not feasible, as this is equivalent to completely characterizing the joint distribution over the representation space. However, by the properties of the training algorithm, we are able to fix the marginal distributions of all the representation bits as $x_k\sim\textrm{Bern}(\beta)$ for some hyperparameter $\beta\in(0,1)$.

Consistent with the training procedure, we encode our database as ${\bX=\{\rmbx_n\}_{n=1}^N\subseteq \{0,1\}^K}$, ${\rmbx_n=\bbf_{\bTheta}(\rmby_n)}$. We then construct a binary tree on the resulting codes.

Given a query $\bar{\rmby}$, we first encode it as ${\barbx=\bbf_{\bTheta}(\bar{\rmby}_n)}$. We then conduct retrieval by traveling down the binary tree with each decision determined by the next bit of $\bar{\bx}$. We define the \emph{$b$-truncated Hamming neighborhood} of $\barbx$ as the set of all examples whose codes share the first $b$ bits of $\barbx$:
\begin{equation}
N^{\mathcal{H}}_b(\barbx)=\left\{ \rmby\in\bY : \norm{\rmbx\up{b}-\barbx\up{b}}_{\mathcal{H}}=0 \right\}\;.
\end{equation}
It is clear that $N^{\mathcal{H}}_{b+1}(\barbx) \subseteq N^{\mathcal{H}}_b(\barbx)\;\forall b=1,\ldots, K-1$. Our retrieval procedure then corresponds to iterating through this family of nested neighborhoods. We expect the cardinality of these to decay approximately exponentially as a function of index. We terminate the retrieval procedure when ${\left| N^{\mathcal{H}}_b(\barbx)\right| < R}$ for some pre-specified terminal neighborhood cardinality, $R\in\mathbb{N}$. It outputs the set $N^{\mathcal{H}}_{b-1}(\barbx)$.

Assuming marginals $x_k\sim\textrm{Bern}(\beta)$ and neglecting dependence between the $x_k$, this results in expected retrieval time~$\mathcal{O}(\frac{\log N/R}{\mathcal{H}(\textrm{Bern}(\beta))})$ where $\mathcal{H}(\textrm{Bern}(\beta))$ is the Bernoulli entropy. If $\beta=\frac{1}{2}$, for example, this reduces to the balanced tree travel time $\mathcal{O}(\log N/R)$.
This retrieval time is logarithmic in the database size $N$, and \emph{independent} of the representation space dimensionality $K$. If one wishes to retrieve a fixed fraction of the dataset, this renders the retrieval complexity also independent of the dataset size.

In many existing retrieval methods, the similarity of two examples is measured by their Hamming distance. Here, similarity is rather measured by the number of leading bits they share. This is consistent with the training procedure, which produces codes with this property by demanding reconstructive ability under code truncation variation. 

\begin{figure}[t!]
    \centering
    \includegraphics[width=0.9\columnwidth]{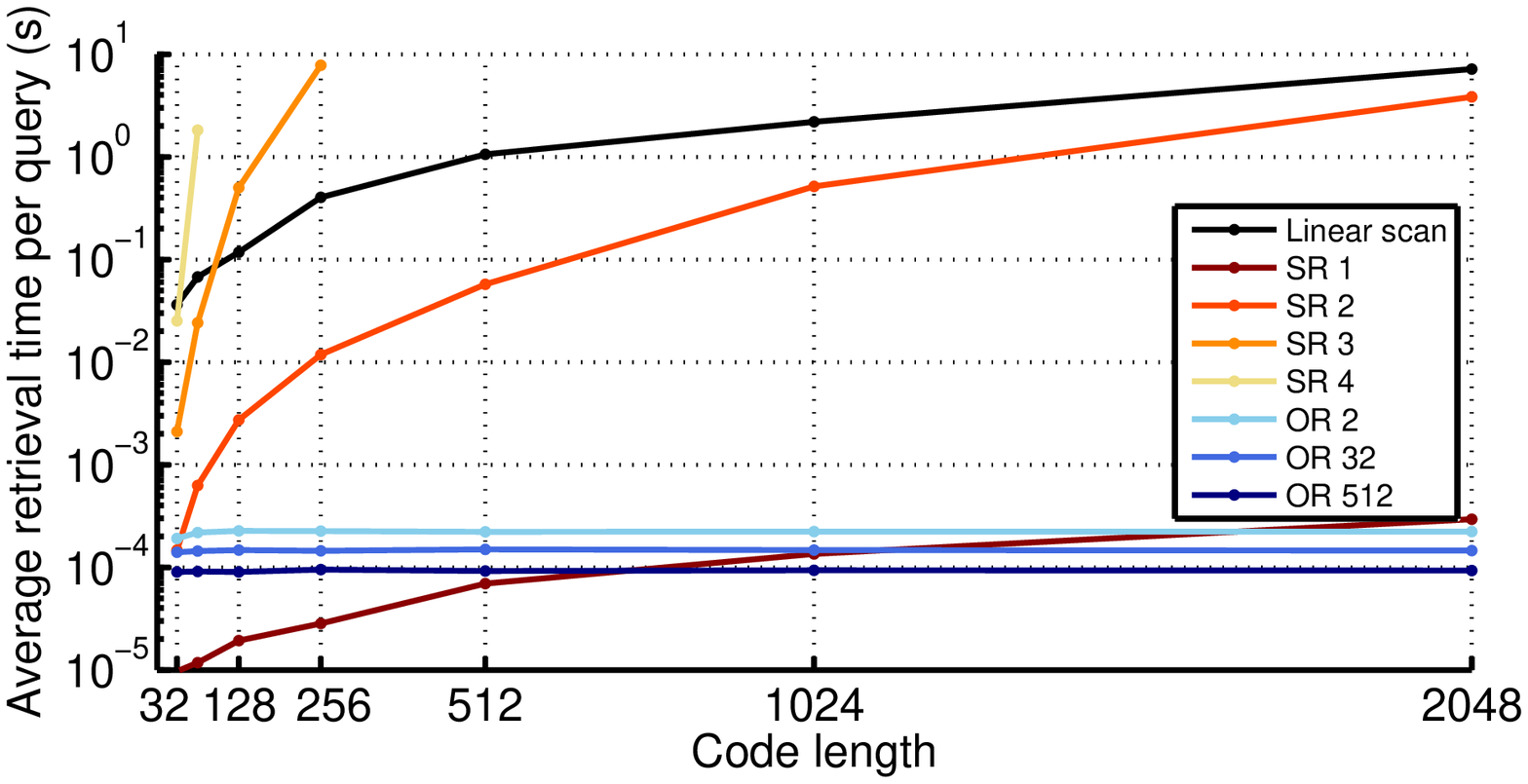}
    \captionv{Empirical timing tests for different retrieval algorithms. SR: semantic retrieval, OR: ordered retrieval. The numbers next to ``SR'' and ``OR'' in the figure legend correspond to the semantic hashing radius and the terminal ordered retrieval neighborhood cardinality, respectively. As the code length increases, Hamming balls of very small radii become prohibitive to scan. Ordered retrieval carries a small fixed cost that is independent of code length.}
    \label{speeds}
\end{figure}

\subsection{Empirical results}
We empirically studied the properties of the resulting codes and data structures in a number of ways. First, we applied ordered retrieval to a toy problem where we trained a tiny 2-16-32-16-2 autoencoder on 2D synthetic pinwheel data (Figure \ref{toy}). Here we can visualize the nesting of neighborhood families for different queries. Note that, as expected, the nested neighborhood boundaries are orthogonal to the direction of local variation of the data. This follows from the model's reconstruction loss function.

We then trained on 80MTI a binarized nested dropout autoencoder with layer widths 3072-2048-1024-512-1024-2048-3072 with $L_1$ weight decay and invariance regularization (see Section \ref{implementation}). We chose~${p_B\cd \sim \textrm{Geom}(0.97)}$ and the binarization quantile~${\beta=0.2}$.

Empirical retrieval speeds for various models are shown in Figure \ref{speeds}. We performed retrieval by measuring Hamming distance in a linear scan over the database, and by means of semantic hashing for a number of radii. We also performed ordered retrieval for a number of terminal neighborhood cardinalities. Although semantic hashing is independent of the database size, for a radius greater than 2 it requires more time than a brute force linear scan even for very short codes. In addition, as the code length increases, it becomes very likely that many queries will not find any neighbors for any feasible radii. It can be seen that ordered retrieval carries a very small computational cost which is independent of the code length. Note that each multiplicative variation in the terminal neighborhood size $R$, from 2 to 32 to 512, leads to a constant shift downward on the logarithmic scale plot. This observation is consistent with our earlier analysis that the retrieval time increases logarithmically with~$N/R$.

In Figure \ref{retrievals}, we show retrieval results for varying terminal neighborhood sizes. As we decrease the terminal neighborhood size, the similarity of the retrieved data to the query increases. As more bits are added to the representation in the process of retrieval, the resolution of the query increases, and thus it is better resolved from similar images. 

\begin{figure}[t!]
\centering
\parbox[t]{\columnwidth}{
    \parbox[t]{0.06\columnwidth}{
    \vspace{-0.25in}
    \tiny $N^{\mathcal{H}}_{8}$ \\

    \vspace{0.26in}
    \tiny $N^{\mathcal{H}}_{32}$ \\

    \vspace{0.29in}
    \tiny $N^{\mathcal{H}}_{128}$ \\

    \vspace{0.18in}
    \tiny $N^{\mathcal{H}}_{256}$ 
    }
    \;\parbox[t]{0.94\columnwidth}{\noindent\includegraphics[width=0.94\columnwidth]{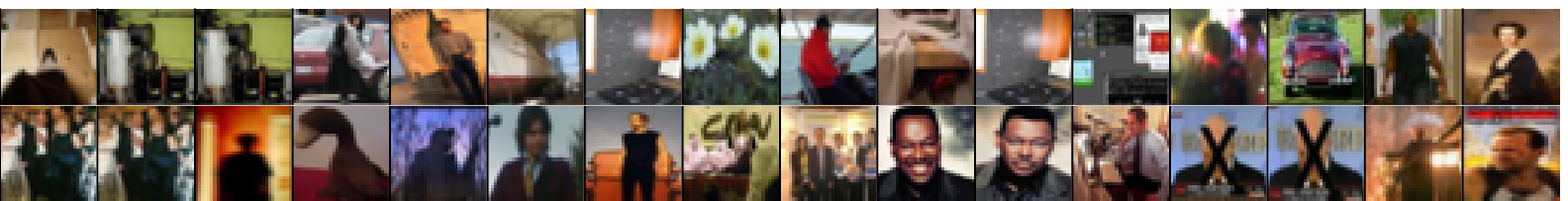}\vspace{0.05in}    
    \includegraphics[width=0.94\columnwidth]{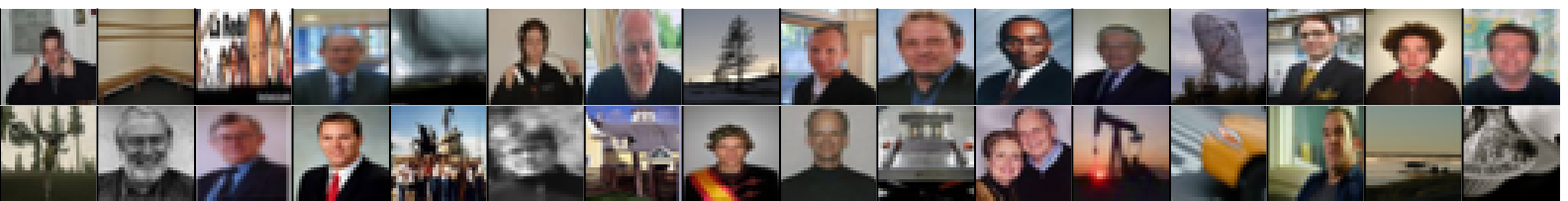}\vspace{0.05in}
    \includegraphics[width=0.94\columnwidth]{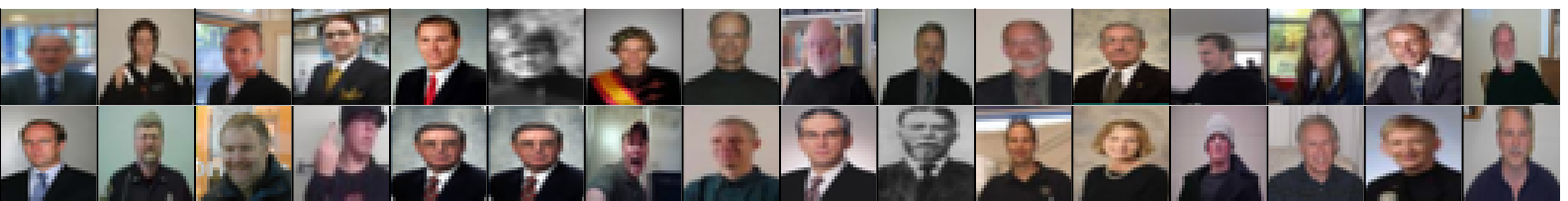}\vspace{0.05in}
    \includegraphics[width=.3525\columnwidth]{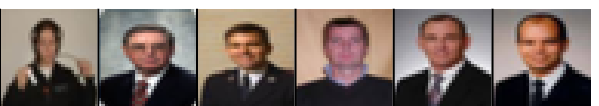}
    \; \parbox[t]{0.06\columnwidth}{\vspace{-0.15in}\tiny $N^{\mathcal{H}}_{311}$}\;\includegraphics[width=.17625\columnwidth]{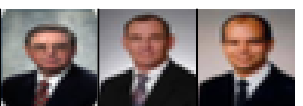}
    \; \parbox[t]{0.06\columnwidth}{\vspace{-0.15in}\tiny $N^{\mathcal{H}}_{465}$}\;\includegraphics[width=.065\columnwidth]{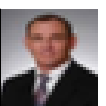}
    }}

    \vspace{0.14in}
\parbox[t]{\columnwidth}{
    \parbox[t]{0.06\columnwidth}{
    \vspace{-0.25in}
    \tiny $N^{\mathcal{H}}_{8}$ \\

    \vspace{0.29in}
    \tiny $N^{\mathcal{H}}_{64}$ \\

    \vspace{0.18in}
    \tiny $N^{\mathcal{H}}_{128}$ 
    }
    \;\parbox[t]{0.94\columnwidth}{\noindent\includegraphics[width=0.94\columnwidth]{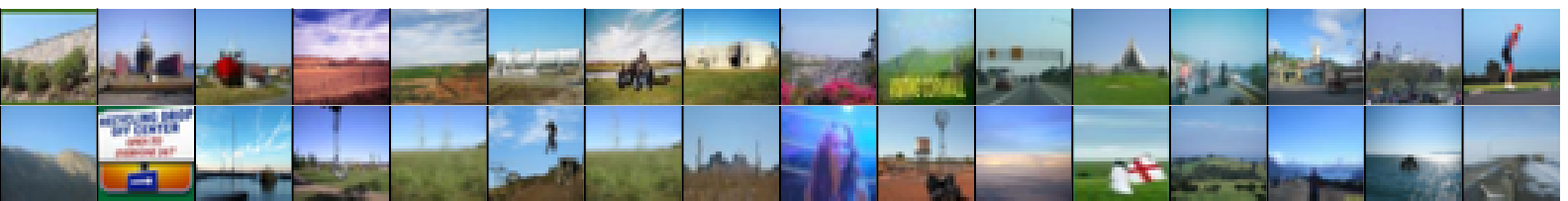}\vspace{0.05in}    
    \includegraphics[width=0.94\columnwidth]{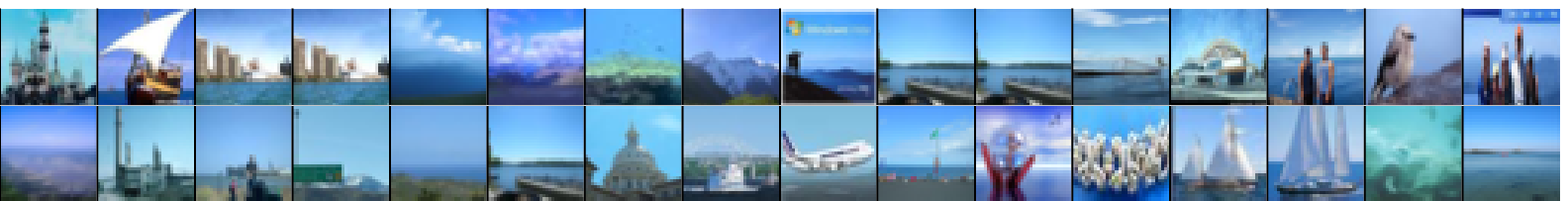}\vspace{0.05in}
    \includegraphics[width=.41125\columnwidth]{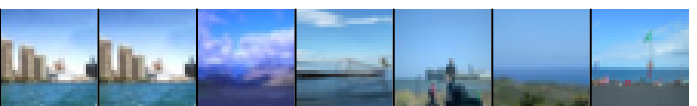}
    \; \parbox[t]{0.06\columnwidth}{\vspace{-0.15in}\tiny $N^{\mathcal{H}}_{158}$}\;\includegraphics[width=.065\columnwidth]{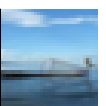}
    }}
    
    \captionv{Retrieval results for different terminal neighborhood cardinalities. Note the increase in retrieval fineness as a function of neighborhood index. Examples presented in the order in which they appear in the dataset. When neighborhood sizes are greater than 32, only the first 32 images in the neighborhood are shown. The last neighborhood contains only the query itself.}
\label{retrievals}
\end{figure}
\label{data_structure}

\section{Adaptive compression}

Another application of ordered representations is continuous-degradation lossy compression systems. By ``continuous-degradation'' we mean that the message can be decoded for any number,~$b$, of bits received, and that the reconstruction error $\mathscr{L}(\rmby,\hat{\rmby}\up{b})$ decreases monotonically with $b$.  Such representations give rise to a continuous (up to a single bit) range of bitrate-quality combinations, where each additional bit corresponds to a small incremental increase in quality. 

The continuous-degradation property is appealing in many situations. First, consider, a digital video signal that is broadcast to recipients with varying bandwidths. Assume further that the probability distribution over bandwidths for the population, $p_B\cd$, is known or can be estimated, and that a recipient with bandwidth $b$ receives only the first $b$ bits of the transmission. We can then pose the following problem: what broadcast signal minimizes the expected distortion over the population? This is formulated as
\begin{equation}
(\bTheta^*,\bPsi^*)=\arg\min_{\bTheta,\bPsi}\bbE_B \left[\mathscr{L}\left(\rmby_n, \hat{\rmby}_{n\downarrow b}\right)\right]\;.\label{comp_eqn}
\end{equation}
This is precisely the optimization problem solved by our model; Equation (\ref{comp_eqn}) is simply a rewriting of Equation (\ref{ndo_problem}). This connection gives rise to an interpretation of $p_B\cd$, which we have set to the geometric distribution in our experiments. In particular, $p_B\cd$ can be interpreted as the distribution over recipient bandwidths such that the system minimizes the expected reconstruction error. 

This intuition in principle applies as well to online video streaming, in which the transmitted signal is destined for only a single recipient. Given that different recipients have different bandwidths, it is acceptable to lower the image quality in order to attain real-time video buffering. Currently, one may specify in advance a small number of fixed encodings for various bandwidths: for example, YouTube offers seven different definitions (240p, 360p, 480p, 720p, 1080p, 1440p, and 2160p), and automatically selects one of these to match the viewer's bitrate. Ordered representations offer the ability to fully utilize the recipient's bandwidth by truncating the signal to highest possible bitrate. Instead of compressing a handful of variants, one needs only to compute the ordered representation once in advance, and truncate it to the appropriate length at transmission time. If this desired bitrate changes over time, the quality could be correspondingly adjusted in a smooth fashion. As above, given a distribution $p_B\cd$ of bandwidths, finding the ordered representation that solves Equation (\ref{comp_eqn}) minimizes the expected distortion over a population of recipients.

\begin{figure}[b]
    \subfigure[\small Reconstructions]
    {\parbox[b]{0.40\columnwidth}{
    \centering
    \includegraphics[width=0.4\columnwidth]{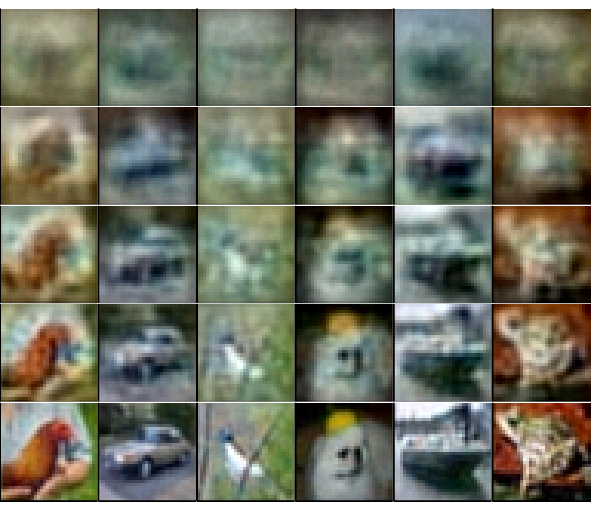}
    \includegraphics[width=0.4\columnwidth]{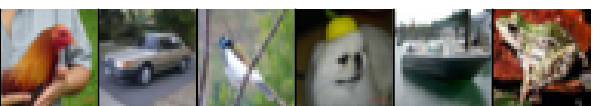}}
    \label{reconst_figures}}
    \hfill
    \subfigure[Reconstruction rates]{
    \centering
    \includegraphics[width=0.5\columnwidth]{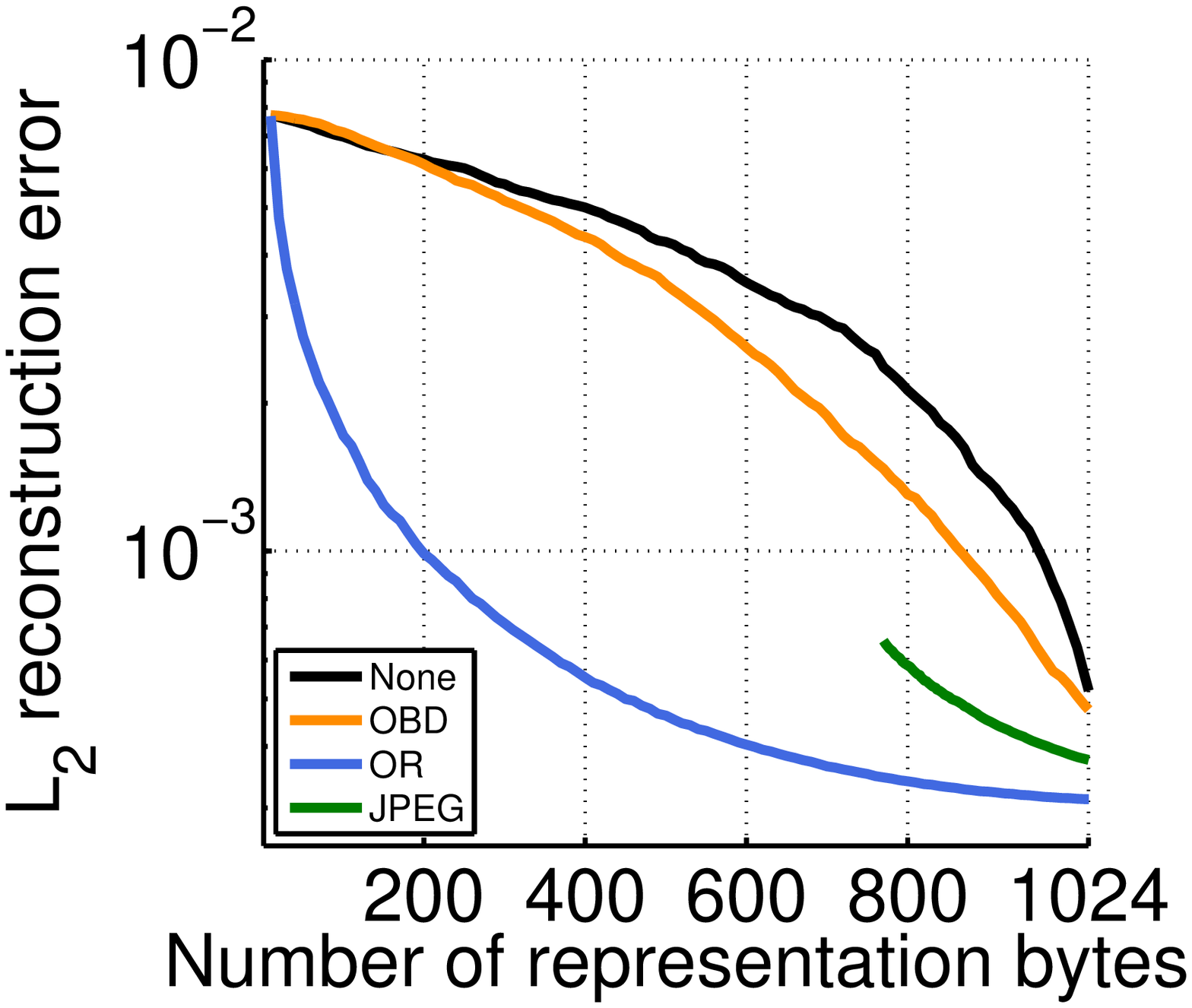}
    \label{reconst_curves}}

    \captionv{Online reconstruction with ordered representations. (a)~ Reconstructions for code lengths 16, 64, 128, 256, and 1024 using a nested dropout autoencoder. The original images are 24576 bits each. (b)~Reconstruction rates as a function of code length for four different truncation techniques: ordered representation, Optimal Brain Damage, a standard autoencoder, and JPEG.}
\end{figure}

\subsection{Empirical results}
In Figure \ref{reconst_figures}, we qualitatively evaluate continuous-degradation lossy compression with ordered representations. We trained a single-layer 3072-1024-3072 autoencoder with nested dropout on CIFAR-10, and produced reconstructions for different code lengths. Each column represents a different image and each row represents a different code length. As the code length increases (downwards in the figure), the reconstruction quality increases. The images second-to-bottom row look very similar to the original uncompressed images in the bottom row (24576 bits each).

Figure \ref{reconst_curves} shows ordered representation reconstruction rates as a function of code length for different approaches to the problem. In addition to the above, we also trained a standard autoencoder with the same architecture but without nested dropout. On this we applied 2 different truncation approaches. The first is a simple truncation on the un-ordered bits. The second is Optimal Brain Damage truncation \citep{optimal_brain_damage}, which removes units in decreasing order of their influence on the reconstruction objective, measured in terms of the first and second order terms in its Taylor expansion. This is a clever way of ordering units, but is disjoint from the training procedure and is only applied retroactively. We also compare with JPEG compression. We use the libjpeg library and vary the JPEG quality parameter. Higher quality parameters result in larger file sizes and lower reconstruction error. Note that JPEG is not suited for the 32x32 pixel images we use in this study; its assumptions about the spectra of natural images are violated by such highly down sampled images which have lost significant low frequency content. When the quality is extended to its maximum, the loss is approximately 0.0003, in the same units (not shown).
\label{compression}

\section{Discussion and future work}

We have presented a novel technique for learning representations in which the dimensions have a known ordering.  This procedure is exactly equivalent to PCA for shallow autoencoders, but can generalize to deep networks as well.  This enables learned representations of data that are adaptive in the sense that they can be truncated with the assurance that the shorter codes contain as much information as possible.  Such codes are of interest in applications such as retrieval and compression.

The ordered representation retrieval approach can also be used for efficient supervised learning. Namely, it allows performing $k$-nearest-neighbors on very long codes in logarithmic time in their cardinality. This idea can be combined with various existing approaches to metric learning of kNN and binarized representations \citep{hamming_learning,nonlin_embedding,distance_learning}.  The purely unsupervised approaches we have described here have not been empirically competitive with state of the art supervised methods from deep learning.  We are optimistic that nested dropout can be meaningfully combined with supervised learning, but leave that for future work.

We also note that ordered representations provide insight into how one might practically train models with an infinite number of latent dimensions, in the spirit of Bayesian nonparametric methods.  For example, the distribution $p_B\cd$ can be chosen to have infinite support, while having finite mean and variance.  Finally, the nesting idea can be generalized to more complicated dependency structures, such as those described in \citet{recursive_cardinality}.

\label{discussion}

\paragraph{Acknowledgements}

We are grateful to Hugo Larochelle for insightful conversations. This work was partially funded by DARPA Young Faculty Award N66001-12-1-4219.

\appendix
\section*{Appendix A. Proofs for Section \ref{pca}}\label{proof_section}

\setcounter{thm}{0}
\setcounter{lemma}{0}

\begin{thm}\label{everysoln}
Every optimal solution of the nested dropout problem is necessarily an optimal solution of the standard autoencoder problem.
\end{thm}

\begin{proof}
Let the nested dropout autoencoder be of latent dimension $K$. Recall that the nested dropout objective function in Equation (\ref{ndo_objective}) is a strictly positive mixture of the $K$ different $b$-truncation problems. As described in Subsection \ref{defns}, an optimal solution to each $b$-truncation must be of the form $\bX_b^* = \bT_b \bSigma\up{b} \bR^T, \bGamma_b^* = \bQ\up{b}\bT^{-1}_b$ for some invertible transformation $\bT_b$. We note that the PCA decomposition is a particular optimal solution for each $b$ that is given for the choice $\bT_b=\mathbb{I}_b$. As such, the PCA decomposition exactly minimizes \emph{every} term in the nested dropout mixture, and therefore must be a global solution of the nested dropout problem. This means that every optimal solution of the nested dropout problem must exactly minimize every term in the nested dropout mixture. In particular, one of these terms corresponds to the $K$-truncation problem, which is in fact the original autoencoder problem. 
\end{proof}

Denote $\bT\up{b}=\trJ{K}{b}\bT\trJ{K}{b}^T$ as the $b$-th leading principal minor and its its bottom right corner as $t_b=T_{bb}$. 

\begin{lemma}\label{lemma}
Let $\bT\in\reals^{K\times K}$ be commutative in its truncation and inversion. Then all the diagonal elements of $\bT$ are nonzero, and for each $b=2,\ldots,K$, either $\bA_b=\textbf{0}$ or $\bB_b=\textbf{0}$.
\end{lemma}

\begin{proof}
We have $\det\bT\up{b} = \det\bT\up{b-1}\det(t_b-\bB_b \bT\up{b-1}^{-1}\bA_b)\neq 0$ since $\bT\up{b-1}$ is invertible. Since $\bT\up{b-1}$ is also invertible, then $t_b-\bB_b \bT\up{b-1}^{-1}\bA_b\neq 0$. As such, we write $\bT\up{b}$ in terms of blocks $\bT\up{b-1}, \bA_b, \bB_b, t_b$, and apply blockwise matrix inversion to find that $\bT\up{b-1}^{-1}=\bT\up{b-1}^{-1}+\bT\up{b-1}^{-1}\bA_b(t_b-\bB_b\bT\up{b-1}^{-1}\bA_b)^{-1}\bB_b\bT\up{b-1}^{-1}$ which reduces to $\bA_b \bB_b=\textbf{0}$. Now, assume by contradiction that $t_b=0$. This means that either bottom row or the rightmost column of $\bT\up{b}$ must be all zeros, which contradicts with the invertibility of $\bT\up{b}$.
\end{proof}

\begin{thm}\label{soln_constraints}
Every optimal solution of the nested dropout problem must be of the form
\begin{eqnarray}
\bX^* &=& \bT \bSigma \bR^T\\
\bGamma^* &=& \bQ \bT^{-1}\;,
\end{eqnarray}
for some matrix $\bT\in\reals^{K\times K}$ that is commutative in its truncation and inversion.
\end{thm}

\begin{proof}
Consider an optimal solution $\bX^*, \bGamma^*$ of the nested dropout problem. For \emph{each} $b$-truncation, as established in the proof of Theorem \ref{everysoln}, it must hold that 
\begin{eqnarray}
\bX_b^* &=& \bT_b \trJ{K}{b}\bSigma\bR^T \\
\bGamma_b^* &=& \bQ \trJ{K}{b}^T\bT^{-1}_b\;.
\end{eqnarray}
However, it must also be true that $\bX_b=\bX\up{b}, \bGamma_b=\bGamma\up{b}$ by the definition of the nested dropout objective in Equation (\ref{ndo_objective}). The first equation thus gives that $\bT_b \trJ{K}{b}=\trJ{K}{b} \bT_K$, and therefore $\bT_b=\trJ{K}{b} \bT_K \trJ{K}{b}^T=\bT\up{b}$. This establishes the fact that the optimal solution for each $b$-truncation problem simply draws the $b$-th leading principal minor from the \emph{same} ``global'' matrix $\bT:=\bT_K$. The second equation implies that for every $b$, it holds that $\trJ{K}{b}\bT^{-1}\trJ{K}{b}^T = (\trJ{K}{b}\bT\trJ{K}{b}^T)^{-1}$ and as such $\bT$ is commutative in its truncation and inversion.
\end{proof}

\begin{thm}
Under the orthonormality constraint $\bGamma^T\bGamma=\mathbb{I}_K$, there exists a unique optimal solution for the nested dropout problem, and this solution is exactly the set of the $K$ top eigenvectors of the covariance of $\bY$, ordered by eigenvalue magnitude. Namely, $\bX^*=\bSigma\bR^T, \bGamma^*=\bQ$.
\end{thm}

\begin{proof}
The orthonormality constraint implies $(\bT^{-1}\bQ)^T\bQ\bT^{-1}=\mathbb{I}_K$ which gives $\bT^T=\bT^{-1}$. Hence every row and every column must have unit norm. We also have have that for every $b=1,\ldots,K$
\begin{eqnarray}
\bT\up{b}^T&=&(\trJ{K}{b}\bT\trJ{K}{b}^T)^T\\
&=&\trJ{K}{b}\bT^T\trJ{K}{b}^T \\
&=&\trJ{K}{b}\bT^{-1}\trJ{K}{b}^T \\
&=&(\trJ{K}{b}\bT\trJ{K}{b}^T)^{-1}\\
&=&\bT\up{b}^{-1}
\end{eqnarray}
where in the last equation we applied Lemma \ref{lemma} to Theorem \ref{soln_constraints}. As such, every leading principal minor is also orthonormal. For the sake of contradiction, assume there exist some $m,n,m\neq n$ such that $T_{mn}\neq 0$. Without loss of generality assume $m<n$. Then $\sum_{p=1}^{n-1}T_{mp}^2<1$, but this violates the orthonormality of $\bT_{n-1}$. Thus it must be that the diagonal elements of $\bT$ are all identically 1, and therefore $\bT=\mathbb{I}_K$. The result follows.
\end{proof}

\bibliography{ordered}
\bibliographystyle{icml2014}

\end{document}